\pdfoutput=1

\documentclass[journal]{IEEEtran}

\ifCLASSINFOpdf
\else
\fi
\usepackage{amsthm}
\usepackage{amssymb}  

\usepackage{graphics} 
\usepackage{float,url,enumerate}
\usepackage{graphicx}
\usepackage{subcaption}

\usepackage{algorithm}
\usepackage{algorithmic}

\usepackage[usenames, dvipsnames]{color} 

\newcommand{\etal}{\textit{et. al. }} 

\usepackage{xspace}
\makeatletter
\DeclareRobustCommand\onedot{\futurelet\@let@token\@onedot}
\def\@onedot{\ifx\@let@token.\else.\null\fi\xspace}
\def\eg{\emph{e.g}\onedot}

\def\etal{\emph{et al}\onedot}
\makeatother

\newcommand{\R}{\mathbb{R}} 
\newcommand{\SO}{SO(3)}

\renewcommand{\P}{{\bf P}} 
\newcommand{\Q}{{\bf Q}} 
\newcommand{\G}{{\bf G}} 
\newcommand{\V}{{\bf V}} 
\newcommand{\qobj}{q_{\rm obj}}
\newcommand{\pobj}{p_{\rm obj}}
\newcommand{\qgrp}{q_{\rm grp}}
\newcommand{\pgrp}{p_{\rm grp}}
\newcommand{\pfingerl}{p_{\rm{f,left}}}
\newcommand{\pfingerr}{p_{\rm{f,right}}}

\newcommand{\ui}{^{({\rm i})}} 
\newcommand{\uf}{^{({\rm f})}} 
\newcommand{\uj}{^{({\rm j})}} 
\newcommand{\uk}{^{({\rm k})}} 

\newtheorem{theorem}{Theorem}

\begin{document}
%
\title{Reorienting Objects in 3D Space Using Pivoting}
\author{Yifan~Hou,~
        Zhenzhong~Jia,~
        and~Matthew~T.~Mason,~\IEEEmembership{Fellow,~IEEE}

\thanks{The authors are with the Robotics Institute, Carnegie Mellon University, Pittsburgh, PA 15213, USA.
        {\tt\small \{yifanh, zhenzjia\}@cmu.edu, matt.mason@cs.cmu.edu}}}

\maketitle

\begin{abstract}
We consider the problem of reorienting a rigid object with arbitrary known shape on a table using a two-finger pinch gripper. Reorienting problem is challenging because of its non-smoothness and high dimensionality. In this work, we focus on solving reorienting using pivoting, in which we allow the grasped object to rotate between fingers. Pivoting decouples the gripper rotation from the object motion, making it possible to reorient an object under strict robot workspace constraints. We provide detailed mechanical analysis to the 3D pivoting motion on a table, which leads to simple geometric conditions for its stability. To solve reorienting problems, we introduce two motion primitives: \textit{pivot-on-support} and \textit{roll-on-support}, and provide an efficient hierarchical motion planning algorithm with the two motion primitives to solve for the gripper motions that reorient an object between arbitrary poses.
To handle the uncertainties in modeling and perception, we make conservative plans that work in the worst case, and propose a robust control strategy for executing the motion plan.
%
Finally we discuss the mechanical requirements on the robot and provide a "two-phase" gripper design to implement both pivoting grasp and firm grasp. We demonstrate the effectiveness of our method in simulations and multiple experiments. Our algorithm can solve more reorienting problems with fewer making and breaking contacts, when compared to traditional pick-and-place based methods.
\end{abstract}


%
\IEEEpeerreviewmaketitle


\section{INTRODUCTION}
\label{sec:intro}

\IEEEPARstart{T}{here} are many things roboticists can learn from human manipulation, even when humans do not take advantage of the dexterity of the hand. For example, a human using the thumb and index finger in a pinch grasp can easily outperform an industrial robot with a pinch gripper in tasks such as picking and reorienting objects. Although the end effectors are similar, humans are better in at least the following two important ways.
First, humans utilize \textit{extrinsic dexterity} \cite{extrinsic} such as gravity, inertia forces and extrinsic contact forces in addition to joint forces.
Second, human hands are compliant. For example, we can hold and manipulate an object while leaning it on a table without crushing anything. As a result, human hands have a larger repertoire of motions to use.

\subsection{The Reorienting Problem} 
\label{sub:the_reorienting_problem}
The frequent appearance of making and breaking contacts in extrinsic dexterity introduces non-smooth mechanics and discrete decision variables. Together with the high dimensionality of the system (object plus robot), they make it challenging to find solutions fast and reliably for mainstream planning and control methods, such as trajectory optimization, sampling based planning or reinforcement learning.

One such example is object reorienting: quickly move an object between different 3D orientations. The reorienting problem is common in industrial applications such as polishing, soldering and assembling, where an object undergoes multiple procedures in sequence; and also home applications such as tidying a table. In these scenarios, the object may have a complicated shape. The robot needs to decide how to manipulate the object (pushing, grasping, etc), while staying within a constrained workspace.

For robots with simple grippers, traditional methods for reorienting simplify planning by only using pick-and-place motion \cite{lozano1987handey,lozano2014constraint,tournassoud1987regrasping,stoeter1999planning,cho2003complete,wan2015improving}. In pick-and-place, the manipulator rotates the object by grasping it firmly, then places it at a different stable pose. The process may repeat a few times before the goal pose is reached. The planning problem becomes totally kinematic and easier to solve. However, pick-and-place brings limits on the robot motion. For example, the robot gripper must rotate along with the object, which is impossible in a limited workspace and takes unnecessarily long time.

\subsection{Motion Primitives for Reorienting} 
\label{sub:use_motion_primitives}
Although pick-and-place has its limitations, it inspires us to solve the reorienting by decomposing it into smaller, well-defined, solvable problems. These sub-problems describe parametrized motions and are designed to be easily solvable, we call them \textit{motion primitives}. For example, single pick-and-place with one grasp is a motion primitive for reorienting. Once we can solve the motion primitive, we can further solve the original problem by run the motion primitive several times with suitable parameters. This step is high level and has much fewer variables to consider, so the computation time is reduced.

In this work, we design better motion primitives for reorienting objects by observing human motion. Even with a two finger pinch grasp, a human hand can reorient an object more elegantly with less hand and arm motion than pick-and-place reorienting. The key difference is that human hands allow and utilize slips between the object and fingertips. We summarize this behavior as a new motion primitive: \textit{pivot-on-support}. The object is pinch-grasped, but can rotate passively about the grasp axis. Meanwhile the object is in contact with the table under gravity. The motion primitive brings several benefits. Firstly, the gripper motion is decoupled from the rotation of the grasped object, making it possible to perform reorienting under stricter workspace constraints. For the same reason, there exist more efficient solutions for robot motion. Secondly, the contact on the table balance the object weight, so we can use the quasi-static assumption to simplify the modeling. pivot-on-support is more flexible than pick-and-place, yet it is easier to control than those more dynamic motion primitives such as throwing and catching \cite{extrinsic}, pivoting with gravity \cite{spoon2016adaptive} or inertia force \cite{lynchsliding,hou2016robust}.

Note that pivot-on-support is not always stable. A pivot grasp does not form a force closure; the object may break the contact with the table and fall over under gravity. To continue reorienting the object when pivoting is unstable, human hand would stop pivoting and switch to a firm non-slip grasp till pivoting becomes stable again. We use the non-slip grasp mode as another motion primitive for our planner, call it \textit{roll-on-support}. As the name suggests, we still maintain the contact with the table as the object rotates, so as to have a smooth transition between the two motion primitives.

We provide an algorithm to solve reorienting efficiently using the two motion primitives. The algorithm takes as input an object with arbitrary shape, as well as an arbitrary 3D initial pose and goal pose. At low level the algorithm plans a gripper trajectory that alternates between the two motion primitives to move the object, given one grasp location. On top of that, it performs graph search to plan multi-step reorientings and pick the best grasp locations. By analyzing the model of pivoting, we provide simple conditions to quickly check if a pivoting pose is stable, which enables fast computing. To further accelerate the online computation, we off-load grasp point selection and some collision checking off-line. Then our hierarchical approach can find a solution or declare infeasibility for a new problem in a couple of seconds. The off-line computation takes several minutes and only needs to be done once per object.

\subsection{Robustness from Planning, Control and Hardware} 
\label{sub:robustness}
The algorithm needs to know the object shape and mass property, which can be hard to obtain accurately. We do robust planning by modeling the uncertainties as bounded errors. Then the algorithm computes conservative plans that work in the worst case.

Still, robust planning alone is not enough for ensuring successful experiments. Picking a perfect value for an uncertainty bound is a tricky trade-off between robustness and feasibility. Instead of carefully tuning uncertainty parameters, we do planning using rough parameters then execute the motion plan robustly with hybrid force-velocity control, which reduces unexpected slips and robot crashes.
We demonstrate our planning and execution scheme in experiments with several real-life objects.

Our algorithm assumes the gripper hardware to be capable of implementing the two motion primitives. Specifically, in pivoting the object should rotate without translational slips between fingers. In this work we discuss the requirements and design principles for such grippers, and provide one simple "two-phase" finger design that can switch between the two motion primitives.

To demonstrate our method, we firstly compare our method with pick-and-place based reorienting method in simulations, and report statistics including number of solved problems and computation time. Then we demonstrate the motion plans in experiments with multiple objects.

This paper is and extension of our previous work on reorienting \cite{hou2018fast}. The contributions of this work are:
\begin{itemize}
    \item Mechanical analysis of pivoting that extends the planar analysis in \cite{hou2018fast} to 3D motion;
    \item Fast and robust motion planning for reorienting in 3D, where the gripper motion planning problem is formulated as a Quadratic Programming (QP).
    \item Hardware design of a two-phase gripper.
    \item Simulation and experiment verifications.
\end{itemize}

\section{RELATED WORK}
\label{sec:related_work}

\subsection{Generic Motion Planning} 
\label{sub:generic_motion_planning}
Theoretically, many existed generic motion planning methods can handle the problem of manipulating an rigid object with a gripper. However, none of them has satisfying performance.
Sampling based planning methods find a solution by building a tree or a graph of states. They can guarantee the existence of solution given enough time \cite{webb2013kinodynamic}. However, the process is usually too long, especially for a system with more than 6 DOFs (object pose and gripper pose in 3D space).
Trajectory optimization suffers from the non-smoothness and high dimensionality. Gradient doesn't exist, which disables powerful continuous solvers. Many work use linear complimentary constraints (LCP) to model the contact as continuous constraints, however, these constraints make the optimization problem close to ill-conditioned \cite{posa2014direct}. It's sensitive to the choice of initial trajectory, and it is hard to further optimize a solution. Existing methods often rely on overly simplified physics model and only demonstrate results in simulations \cite{mordatch2012contact,mordatch2012discovery,toussaint2018differentiable}, and still take minutes to find a solution. Reinforcement learning has also shown ability to solve a range of manipulation problems \cite{pinto2016supersizing,levine2016end}. However, one basic assumption in reinforcement learning is that sampling actions randomly can explore useful actions. This is only true for tasks that are generally stable. It is almost impossible for a task like flipping an object by pivoting. The sampled action has to stay on the right sub-manifold for some time steps, otherwise the object will drop and the exploration go back to the beginning.

\subsection{Object Reorienting} 
\label{sub:object_re_orienting}
Work on pick-and-place reorienting can date back to 1970s, when Richard Paul used a robot arm and a parallel gripper to pick-and-place a cube in correct orientation \cite{instant_insanity_video1971}. Since 1980s, the mainstream approaches would discretize the whole gripper-object state space with a "grasp-placement table" \cite{lozano1987handey,lozano2014constraint,tournassoud1987regrasping}. The set of grasps were selected off-line. A motion sequence can be obtained by back chaining from the goal during planning. The approach is then adopted and improved in several ways. Stoeter \etal computed the stable placements of the object and the discretized gripper motion, built the list of "Grasp-placement-grasp" triples for searching online \cite{stoeter1999planning}. Cho \etal stores collision-free states in a lookup table, then find intermediate placement incrementally by querying the table \cite{cho2003complete}. Wan \etal \cite{wan2015improving} and Xue \etal \cite{xue2008planning} utilize graph structure to represent feasible gripper and object poses for efficient online searching. Wan \etal also decompose the search of pick-and-place sequence from the search of grasps for better efficiency. Cao \etal \cite{cao2016empirical,cao2016analyzing} extended the graph structure to consider fixtures of different shapes instead of a plain table. In those work, only stable placements are discretized and stored as entries in the table/nodes on the graph. In our previous work, we extended the graph to consider arbitrary initial and goal object poses, yet maintained a concise graph by leverage more computations in the edges.

\subsection{Regrasping} 
\label{sub:regrasping_and_re_orienting}
Closely related to reorienting, \textit{regrasping} problems focus on changing the pose of an object with respect to the gripper, while reorienting sets the goal in world frame. Since the transformation between the world frame and the gripper frame is usually known, solutions to the two problems can be used for each other. Dafle \etal demonstrated by hand-coded trajectories that a robot with a simple gripper is able to perform dexterous regraspings by utilizing the extrinsic force resources \cite{extrinsic}.  Dafle and Rodriguez then used LCP trajectory optimization for planning continuous regrasping motions \cite{prehensile} with extrinsic contacts, as well as sampling-based planning for picking discrete modes \cite{chavan2017sampling,chavan2018hand}. Cruciani \etal extended the graph structure in \cite{wan2015improving} to consider both release-and-regrasp and prehensile pushing style in hand manipulation. They further extended the method for incremental data to do in hand manipulation with unknown objects \cite{cruciani2019hand}.




\subsection{Pivoting as a Motion Primitive} 
\label{sub:motion_primitives_for_re_orienting}
Pivoting has been used as a motion primitive with several kinds of extrinsic actuations, including inertia force \cite{hou2016robust,lynchsliding}, gravity \cite{brock,rao1996complete,Sintov_Robotic_2016,spoon2016adaptive} as well as extrinsic contact forces \cite{Inaba_pivoting,holladay2015pivot,terasaki1998motion,yoshida2010pivoting}. Using inertia force or gravity with grip force control can improve the speed of object motion, however, it poses high requirement on the bandwidth of the robot itself \cite{hou2016robust,spoon2016adaptive,Sintov_Robotic_2016}. In \cite{rao1996complete}, the robot lifted up the object and let the object rotate passively under gravity to the desired pose, assuming a suitable grasp along the line of gravity exists. Closely related to our work, Holladay \cite{holladay2015pivot} and Terasaki \cite{terasaki1998motion} used pivot-on-support for reorienting objects. In both work they analyzed pivoting as a swinging between two stable placements, which require good knowledge of the object inertia model and suitable grasp location.

\subsection{Gripper Design for Pivoting} 
There are many different ways to implement pivoting. A parallel/pinch gripper can do pivoting by regulating inertia force \cite{lynchsliding} and gripping force \cite{spoon2016adaptive}. However, it's more robust to do pivoting using a customized finger design that makes it easy to rotate without slipping. The Freddy II robot is probably the first to use a rotating mechanism with active motor control on its fingertips \cite{ambler1973versatile}. To make the design compact, passive rotation mechanisms were widely used. Several work used a pair of point contacts on the object to rotate it without slip \cite{Inaba_pivoting,holladay2015pivot,chavan2015two-phase,chavan2018pneumatic}. People also made face contacts rotatable by adding a rotation shaft to the fingertips of a parallel gripper \cite{carlisle1994pivoting,terasaki1998motion,hou2016robust}, such that the grasped object can rotate along with the fingertips about the grasp axis.

As we mentioned before, pivoting along is not enough for reorienting objects. The gripper must be able to switch between pivoting and firm grasping. The transition can be done by deforming the soft fingertip \cite{spoon2016adaptive}, but it is hard to model the change of friction force precisely.
Terasaki \etal \cite{terasaki1998motion} designed a two-phase finger by surrounding a rotational fingertip with a fixed fingertip. The fixed fingertip will contact the object and stop the rotation if the grasp distance decreases.
Dafle \etal \cite{chavan2015two-phase,chavan2018pneumatic} proposed another two-phase fingertip design that can switch between a point contact and a face contact, so as to transit between pivoting and firm grasping. In this work we provide a new design that combine some of the merits of previous designs.




\section{PROBLEM DESCRIPTION}
\label{sec:problem_description}
Consider a rigid object described by a 3D mesh. Denote $^O\P\in \R^{3\times N}$ as the points on its convex hull measured in the object frame $O$. Its center of mass (COM) in $O$ frame is given by $^OC\in \R^{3}$.
Denote $^W\qobj\in \SO, ^W\pobj\in \R^3$ as the position and orientation of the object in world frame $W$.
Consider a parallel gripper. The origin of the gripper frame is fixed at the middle point between the two fingertips. The $Z$ axis points towards the palm, $X$ axis parallels the grasp axis. We use $^W\qgrp \in \SO$ and $^W\pgrp$ to describe the pose of the gripper frame.

Consider a parallel gripper. Denote $\G$ as a set of pre-computed grasping locations for the object. We ignore grasp points on corners and only consider grasp points on facets, so a parallel firm grasp is possible. There are plenties of methods available for grasp sampling \cite{saut2010planning,wan2015improving}. Each grasp location $^Og \in\G$ is distinguished by two grasp points on the object surface, $^O\pfingerl$ and $^O\pfingerr$. We assume that the grasp points do not change during switching between pivoting and rolling, and that the object rotation axis during pivoting is the line between the two grasp points.
\begin{figure}[h]
    \centering
    \includegraphics[width=0.35\textwidth]{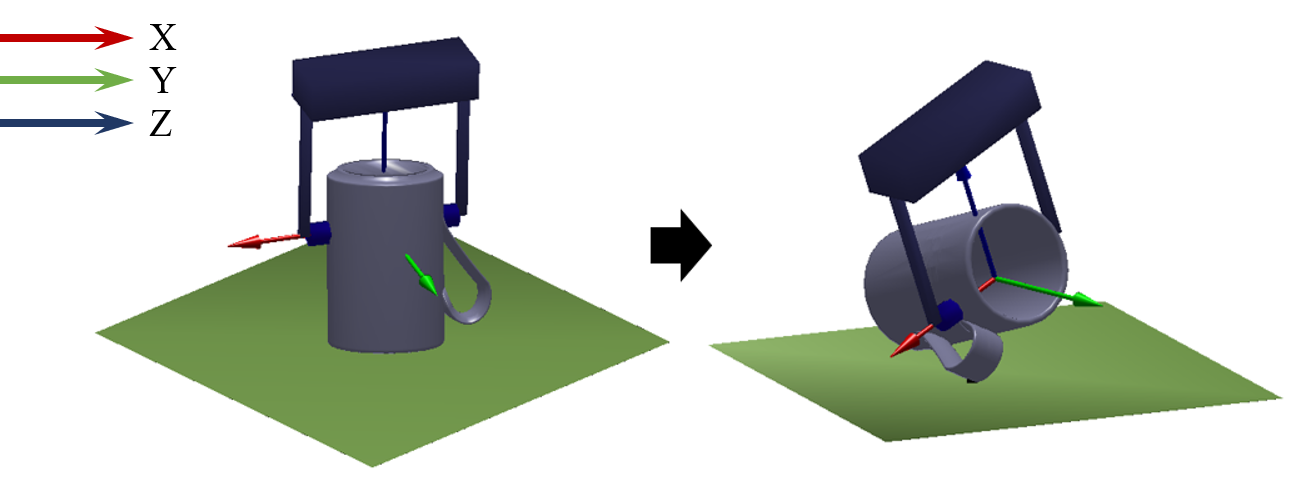}
    \caption{Example of a 3D reorienting problem. The arrows show the axes of the gripper frame.}
    \label{fig:problem}
\end{figure}

A reorienting problem is defined as follows. Given user-defined initial and final object poses $^W\qobj\ui, ^W\qobj\uf \in \SO, ^W\pobj\ui, ^W\pobj\uf \in \R^3$, plan a sequence of gripper motions to move the object from the initial pose to the final pose using the grasps from $\G$. An illustration of the problem is shown in Fig. \ref{fig:problem}.

\begin{figure}[t]
    \centering
    \includegraphics[width=0.45\textwidth]{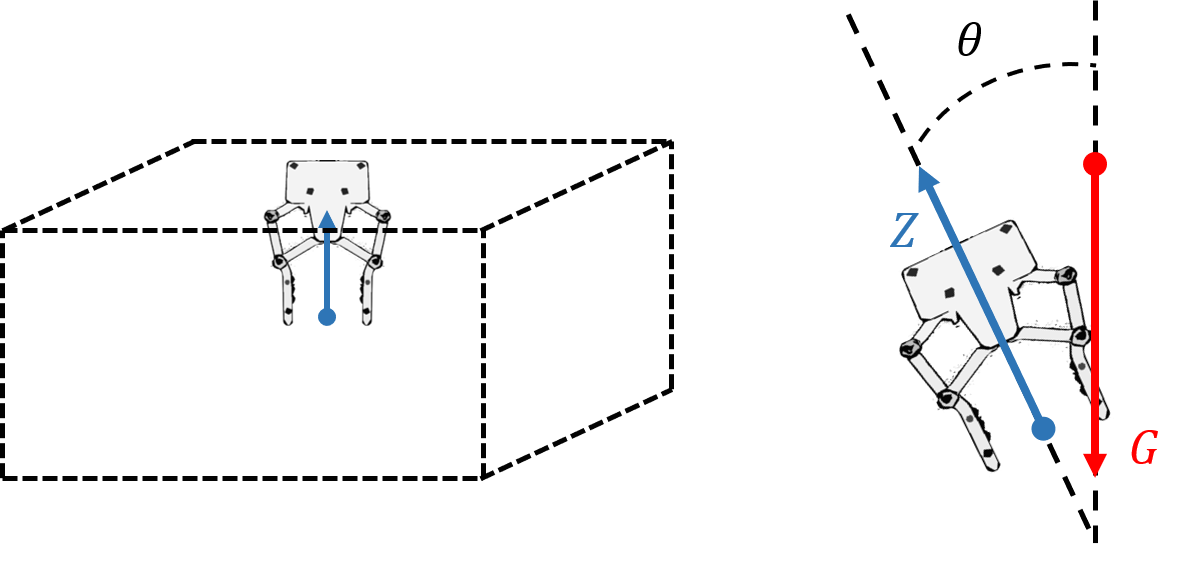}
    \caption{The bounding box and tilting angle workspace constraints. Blue arrow shows the Z axis of gripper frame, red arrow shows the direction of gravity.}
    \label{fig:dexterous_workspace}
\end{figure}

Several constraints need to be considered in the planning problem. First, there should be no collision between the object and any part of the gripper other than the fingertip. Second, the gripper must stay in a preferred workspace so as to avoid  singularities as well as collision with the environments. The workspace constraint we use in this paper is defined based on a common picking system, as shown in Fig. \ref{fig:dexterous_workspace}: the gripper position $^W\qgrp$ must stay within a bounding box described by two corner points $^Wp^{(\min)}$ and $^Wp^{(\max)}$, while the gripper tilting angle $\theta$ has an upper limit $\theta_{\max}$. Zero tilting angle is the "palm down" pose, during which the $Z$ axis of the grasp frame points upward. Other form of workspace constraints can be incorporated into our algorithm as long as they are linear on the gripper pose.


\section{Modeling for Reorienting}
\label{sec:modeling}
In this section we introduce the motion primitives for reorienting, and provide thorough mechanical analysis. In the end we provide two theorems about the stability of pivoting and the contact mode between the object and the table, which will be the foundation of the planning method in the next section.

We start with a list of assumptions made in our modeling.
\begin{enumerate}
  \item The scenario includes one object on a horizontal table, and one robot with a parallel gripper.
  \item The object motion is quasi-static, i.e. inertia forces are negligible.
  \item In pivoting, the object only rotates about the grasp axis in gripper frame.
  \item The friction between the object and the table can be modeled by Coulomb friction and stiction.
\end{enumerate}
%
We model modeling uncertainties including friction, object COM position, actual grasp locations with bounded error. The workspace constraints make sure the robot will not hit singularity and have no collision with the environment.

\subsection{Motion Primitives} 
\label{sub:motion_primitives}
\begin{figure}[h]
    \centering
    \includegraphics[width=0.4\textwidth]{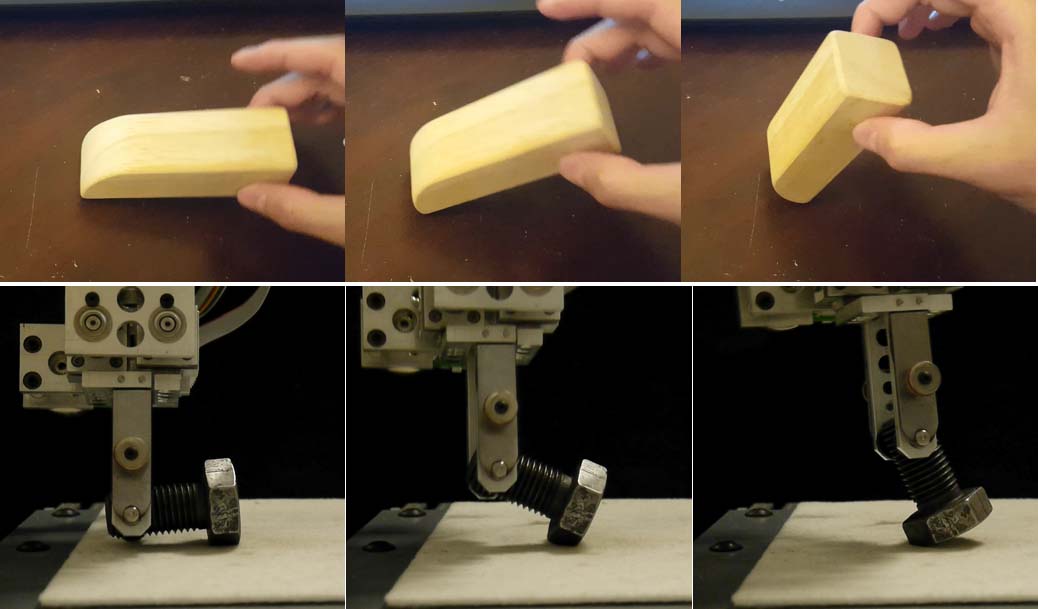}
    \caption{Human and robot doing pivot-on-support.}
    \label{fig:pivoting_motion_primitive}
\end{figure}
We learn two motion primitives from human pinch grasps. The first one is \textit{pivot-on-support} (Fig. \ref{fig:pivoting_motion_primitive}. The gripper grasps two antipodal points on the object surface, while allowing the object to rotate passively about the grasp axis. The grasped object also maintains contact(s) with the table under gravity. We make no assumptions on whether the object is sticking or slipping on the table. Given a gripper pose, the possible object poses have one rotational freedom. If we further know the contact location(s) between the object and table, we can determine the unique object pose kinematically.

Note that pivoting can be unstable. The contact between the object and the table could break under gravity for certain configurations. If we continue to rotate the object, it will fall over and break the quasi-static assumption. Looking for a stable rotation trajectory is tricky. Instead, we continue rotating the object when pivoting is unstable, but with a firm grasp instead to avoid falling. This is our second motion primitive \textit{roll-on-support} in which the gripper grasp the object firmly so that the object motion follows the gripper exactly. We still let the object contact the table, so that we can switch between the two motion primitives easily. To summarize, we let the object rotate directly towards the goal, during which the gripper uses pivot-on-support whenever possible, and switches to roll-on-support otherwise.
In the following, we use \textit{pivoting} and \textit{rolling} to refer to the two motion primitives.

In Sec. \ref{sec:evaluation}, we provide the fingertip design we used to switch between the two motion primitives. The choice of motion primitive, $m\in\{{\rm rolling}, {\rm pivoting}\}$, is a discrete decision variable that the planning algorithm needs to compute.


\subsection{Mechanics of Pivot-on-support}
\label{sub:mechanics_of_pivoting}

%
\begin{figure}[ht]
    \centering
    \includegraphics[width=0.4\textwidth]{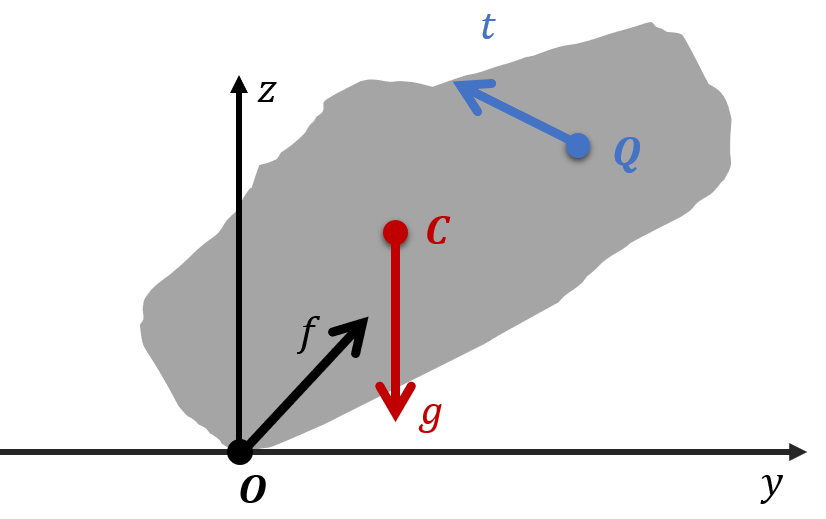}
    \caption{Forces on the object during pivoting, viewing from the direction of the grasp axis. Points $O$, $C$, $Q$ are the contact point, center of mass and grasp point, respectively. Forces $f$, $g$, $t$ are contact reaction force, gravity and force from gripper, respectively.}
    \label{fig:pivoting_mechanics}
\end{figure}
Although it's straightforward to compute a gripper pose from a desired object pose, in order to use pivoting we still need to consider two more things about its mechanics. Firstly we need to know when is pivoting stable/unstable, so that we can choose the right  motion primitive to use. Secondly we need to know when the object can slide on the table, during which the robot gripper may translate freely in the horizontal directions. This make it easier for the robot to stay in a confined workspace.

To analyze the stability of 3D pivoting, consider the object motion in the plane perpendicular to grasp axis, as shown in Fig.~\ref{fig:pivoting_mechanics}. Call this plane the \textit{pivoting plane}. Assign a frame to the pivoting plane by setting the origin $O$ at the contact point between the object and the table, align the X axis with the grasp frame X axis, align the Y axis with the table. Z axis is determined by the right hand rule. Call this frame the \textit{pivoting plane frame}.

We do not assume the plane is fixed during pivoting (as we did before in \cite{Hou2018FastPlanning}), i.e. the gripper could have any 3D motion. Previous works on mechanics of pivot-on-support restricted the points O, C and Q to be co-linear, and focused on computing the required pushing force \cite{Inaba_pivoting}. We remove this restriction and perform stability analysis for all possible pivoting scenarios.

For a quasi-static system, at any time the force system $(f, g, t)$ must be in equilibrium ($O_y = 0$):
\begin{equation}
  \label{eq:force}
  \begin{array}{rll}
  f_y + t_y& = &0, \\
  f_z + t_z& = &g, \\
  C_yg+Q_zt_y& = &Q_yt_z.
  \end{array}
\end{equation}
Inertia force is ignored due to the quasi-static assumption. Eliminate term $t_y, t_z$ we have
\begin{equation}
\label{eq:friction}
  f_z=\frac{Q_z}{Q_y}f_y + (1-\frac{C_y}{Q_y})g.
\end{equation}
Denote $\mu$ as the Coulomb friction coefficient between the object and the table. The contact reaction force $f$ must stay within the friction cone:
\begin{equation}
  \label{eq:coulumbs_law}
  |f_y|\le \mu' f_z.
\end{equation}
Note here we use the symbol $\mu'$ instead of $\mu$ to describe the influence of a tilted pivoting plane. The cone described by $\mu'$ is the projection of the actual 3D friction cone onto the pivoting plane. In this work we assume the gripper tilting angle $\theta$ satisfies $\theta \le \tan^{-1}\frac{1}{\mu}$,
otherwise the projection would no longer be a cone.

The pivoting system is stable as long as a solution exists for the force equilibrium (\ref{eq:friction}) and friction constraint (\ref{eq:coulumbs_law}). Infeasibility means an equilibrium is not possible and the object will topple over.
We can discuss the existence of solutions by looking at the $f_z$ - $f_y$ curve along with the friction cone.
\begin{figure*}[t]
    \centering
    \includegraphics[width=\textwidth]{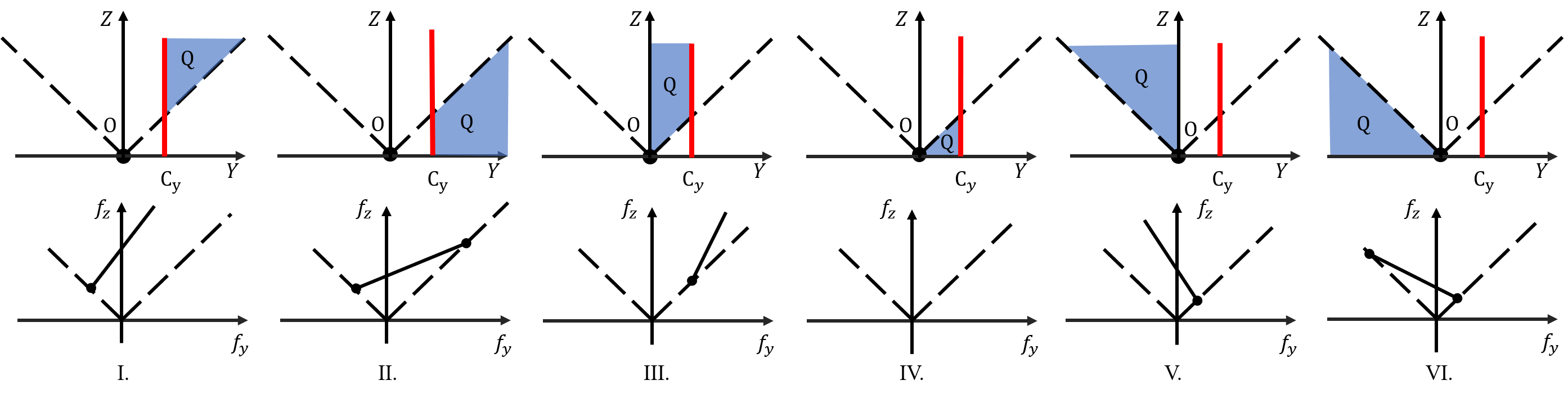}
    \caption{All six possible scenarios and the corresponding solution sets of $f$. In the top figure for each scenario, $O$ is shown as the black dot, possible locations of $C$ are drawn by the red line, possible locations of $Q$ are denoted by the blue shaded area. Solid lines in the bottom figure show the possible value of $f_y, f_z$ under force equilibrium. Dashed lines are the friction cones.}
    \label{fig:pivoting_situations}
\end{figure*}

Depending on the locations of $O, C$ and $Q$ in the pivoting plane, we have six possible situations illustrated by the six columns in Fig.~\ref{fig:pivoting_situations}. In each column, the top figure shows the positions of $C$ and $Q$ relative to $O$ for that situation. The bottom figure shows the corresponding solution set of $f$. In all figures, we also show the friction cones in black dashed line.

When we read Fig.~\ref{fig:pivoting_situations}, it's important to remember the gripper moves in 3D space. The contact force during sliding may project to anywhere inside the friction cone in the pivoting plane, which implies that although equation \ref{eq:coulumbs_law} still holds, a strict inequality no longer means the contact can not slide.

The robot actions, i.e. the gripper velocities, do not directly control the force $t$. Instead, the gripper velocity command only determines in what motion the contact $O$ tend to move. Then we can use Fig.~\ref{fig:pivoting_situations} to tell what will happen.

In scenario I and V, if the contact $O$ tend to move towards $Q$ (roughly means the gripper pulling away from $O$), from the figure we know the friction $f_y$ will be at most the intersection of the friction cone and the line of $f$. So the contact is sliding. However, if the contact $O$ tend to move away from $Q$ (roughly means the gripper push against the contact $O$), from the figure we know the friction force exists and never intersects the friction cone, which means the contact may be sticking. If the contact $O$ is not moving, the force $f$ could be anywhere on the solid line above the friction cone, so the contact may be sticking. If the gripper only moves in the pivoting plane, all the ``may be sticking'' here becomes strictly sticking.

Similarly in scenario II and VI, if the contact $O$ tend to move, the object will slide with limited friction. If the contact $O$ does not move, the object may be sticking.

In scenario III, if the contact $O$ tend to move towards $Q$, there is no solution in equilibrium. If the contact $O$ tend to move away from $Q$, contact force $f$ will either be on the only feasible point on the friction cone, or be infeasible, depending on the exact force from gripper. We ignore this stable solution since it is extremely sensitive to the gripper force and very hard to maintain. Finally if the contact $O$ is not moving, a stable, ``may be sticking'' solution exists if the gripper is pushing towards $O$ with enough force. This solution corresponds to the solid line above the friction cone.

In scenario IV, there is no solution in equilibrium no matter how we move the gripper. The analysis above is summarized in table \ref{tab:contact_states}.
\begin{table}[]
\caption{The contact state under different robot motions for each scenario.}
\label{tab:contact_states}
\begin{tabular}{|l|l|l|l|}
\hline
       & O moves towards Q & O moves away from Q   & O is static                                                                  \\ \hline
I, III & Sliding           & Impossible (be stuck) & \begin{tabular}[c]{@{}l@{}}Sliding\\ or sticking\end{tabular}                \\ \hline
II, VI & Sliding           & Sliding               & \begin{tabular}[c]{@{}l@{}}Sliding\\ or sticking\end{tabular}                \\ \hline
III    & Unstable          & Unstable              & \begin{tabular}[c]{@{}l@{}}Sliding or\\ sticking or \\ unstable\end{tabular} \\ \hline
IV     & Unstable          & Unstable              & Unstable                                                                     \\ \hline
\end{tabular}
\end{table}
In all scenarios except III and IV, a solution exists no matter how the gripper moves. These observations lead to the following sufficient condition for stability:
\begin{theorem}
\label{thm:stability}
A stable solution of the pivoting system (\ref{eq:force}) (\ref{eq:coulumbs_law}) exists as long as the location of $Q$, $C$ and $O$ in the pivoting plane satisfies:
\begin{equation}
  \label{eq:stability_1}
  (Q_y-C_y)\cdot(Q_y-O_y) > 0
\end{equation}
i.e. in the $Y$ direction, the gripper position $Q$ is not in between the contact point $O$ and the center of mass $C$.
\end{theorem}
\begin{proof}
The condition corresponds to I, II, V and VI in Fig.~\ref{fig:pivoting_situations}. As per the analysis above,  the force equilibrium constraints have solutions in all these cases.
\end{proof}
Condition (\ref{eq:stability_1}) is not necessary since a solution exists for scenario III. We do NOT use this solution because it is hard to implement. On the one hand the solution requires sticking with certain force, which is hard to implement precisely. Note the condition described in theorem \ref{thm:stability} has no requirement on force. On the other hand it is hard to tell whether we are in III or IV since the friction modeling could have large errors. On the contrary, the condition in theorem \ref{thm:stability} has nothing to do with the friction coefficient between the object and the table.

We can also come up with a sufficient condition for sliding on the table from Fig.~\ref{fig:pivoting_situations}:
\begin{theorem}
\label{thm:sliding}
When the stability condition (\ref{eq:stability_1}) for pivoting is satisfied, the object can slide on the table without sticking
\begin{itemize}
  \item if the grasp position $Q$ is outside of the friction cone;
  \item or if the proposed gripper motion will move the contact $O$ towards $Q$.
\end{itemize}
\end{theorem}
\begin{proof}
Under condition (\ref{eq:stability_1}), $Q$ being outside of friction cone corresponds to scenario II and VI in Fig.~\ref{fig:pivoting_situations}, in which the frictions $f_y$ in both directions are bounded, so sticking is not possible.
In each of I, II, V and VI, the friction $f_y$ is bounded when the gripper moves away from $O$.
\end{proof}
The two theorems make it possible to design simple yet effective planning algorithms.

\subsection{Simplified Mesh Model} 
\label{sub:simplified_mesh_model}
In practise, the location of the contact point $O$ is computed as the lowest point on the object mesh model. So the planning algorithm presented in the next section involves manipulating the convex hull of the object mesh model. A  delicate model with thousands of facets will slow down the algorithm. We instead work with a simplified (triangle counts reduced) mesh model, and plan robustly against the modeling error.

Denote $^W\P, {^W\hat \P}$ as the vertices of the mesh model convex hull and the simplified mesh model convex hull, respectively. Denote $d_H(,)$ as the Hausdorff distance between two meshes:
\begin{equation}
  d_H(\P, \Q) = {\max\{\,\sup _{p\in \P}\inf _{q\in \Q}d(p, q),\,\sup _{q\in \Q}\inf _{p\in \P}d(p, q)\,\}{\mbox{,}}\!}
\end{equation}
In the following we simplify $d_H(^W\P, {^W\hat\P})$ as $d_H$.
The following theorem determines the range of possible locations of the contact point $O$:
\begin{theorem}
\label{thm:mesh_simplification}
Denote $a\in {^W\hat \P}$ as the bottom point on the simplified mesh, i.e.
$$ {a_z} = \mathop {\min }\limits_{\hat p \in {^W{\hat \P}}} {\hat p_z} $$
Then the actual contact point $e\in {^W\P}$ on the original mesh, defined by
$$ {e_z} = \mathop {\min }\limits_{p \in ^W\P} {p_z} $$
must be within the following set:
$$e \in \left\{ {{B_{{d_H}}}({\hat p_i})\ \ |\ \ {\hat p_i} \in {^W\hat\P},\;\;\,{\hat p_{iz}} \le {a_z} + 2{d_H}} \right\} $$
$B_r(p)$ denotes the ball of radius $r$ centered at point $p$.
\end{theorem}
\begin{proof}
The actual height $O_z$ of the table surface is upper bounded by:
$$ O_z \le a_z + d_H. $$
From the definition of Hausdorff distance we know that points on $^W\P$ is bounded in balls of radius $d_H$ centered at points on $^W\hat\P$:
$$^W\P \in \left\{ {{B_{{d_H}}}({p_i})\ \ |\ \ {p_i} \in {^W\hat \P}} \right\}.$$
As a result, the balls that could possibly touch the table are centered at
$$\left\{ {{{\hat p}_i}\ |\ {{\hat p}_i} \in {^W\hat\P},{\hat p_{iz}} \le {O_z} + {d_H} \le {a_z} + 2{d_H}} \right\}. $$
\end{proof}
Theorem~\ref{thm:mesh_simplification} is useful for plan robustly against model simplification error.

There are many off-the-shelf algorithms for mesh simplification in the computer graphics community \cite{kalvin1996superfaces,klein1996mesh}. We use an algorithm similar to the one described in \cite{klein1996mesh} to produce simplified mesh under user-defined Hausdorff distance.


\section{Planning Under Modeling Uncertainties}
\label{sec:planning}
Now we are ready to design a reorienting planning algorithm that is robust to modeling errors.

Using pivoting, a parallel gripper can reorient an object without releasing the object. However, for some difficult reorienting problems, it may still be necessary for the robot to leave the object on some stable placements, change the grasp location then reorient again several times. This involves planning at different levels, we handle them separately.

\subsection{Reorienting Planning Given One Grasp Location} 
\label{sub:planning_with_one_grasp}
We firstly focus on the reorienting problems that can be solved with only one grasp location. Assume the choice of grasp $^Og \in\G$ is given. We now need to solve for a trajectory of continuous gripper motion as well as the discrete choices of modes, subjected to pivoting stability constraints, collision-avoidance (between object and gripper) constraints and robot workspace constraints. Discretize the trajectories into $N$ time steps, the planning problem takes form of a Mixed-Integer-Programming which suffers from combinatorial explosion of modes. However, we can efficiently find good solutions by decomposing the planning into four steps:
%
%
%
\subsubsection{\bf Plan for object rotation}
In the first step we only compute a trajectory of object rotations. To find the shortest path that connects the initial and final object orientations, we perform a spherical linear interpolation(SLERP) with $N$ steps to obtain a evenly-spaced rotation trajectory.

Some workspace constraints are only influenced by the object rotation, \eg the minimal distance between the fingertips and the table. We check the trajectory for these constraints, and declare infeasibility if a violation happens.

\subsubsection{\bf Choose motion primitives}
We choose pivoting over rolling whenever pivoting is stable.
The object orientation determines the relative position of point $O$, $C$ and $Q$, so we can check the stability of pivoting at each time step using theorem~\ref{thm:stability}, as shown in Figure~\ref{fig:pivoting_stability_robust}, left. Note that due to mesh simplification the contact point $O$ has uncertainty that can be computed from theorem~\ref{thm:mesh_simplification}. The measurement uncertainty of the center of mass brings uncertainty into $C$. Perception errors in the object initial pose will pass on to $Q$. To handle these uncertainties, we declare a pose to be stable for pivoting only if theorem~\ref{thm:stability} is satisfied under all possible locations of $O$, $C$ and $Q$, as shown in Figure~\ref{fig:pivoting_stability_robust}, right. The robust version of condition~\ref{eq:stability_1} is:
\begin{equation}
  (Q_y-OC^{(\max)}_y)\cdot(Q_y-OC^{(\min)}_y) > 0,
\end{equation}
where
\begin{equation}
  \begin{array}{*{20}{l}}
  OC^{(\max)}_y = \max\{{\rm range}(O_y), {\rm range}(C_y)\}, \\
  OC^{(\min)}_y = \min\{{\rm range}(O_y), {\rm range}(C_y)\}.
  \end{array}
\end{equation}

In this way, we also handle multiple contacts naturally. If there are more than one contact points between the object and the table in the pivoting plane, all of these points will be considered in ${\rm range}(O_y)$.

\begin{figure}[h]
    \centering
    \includegraphics[width=0.5\textwidth]{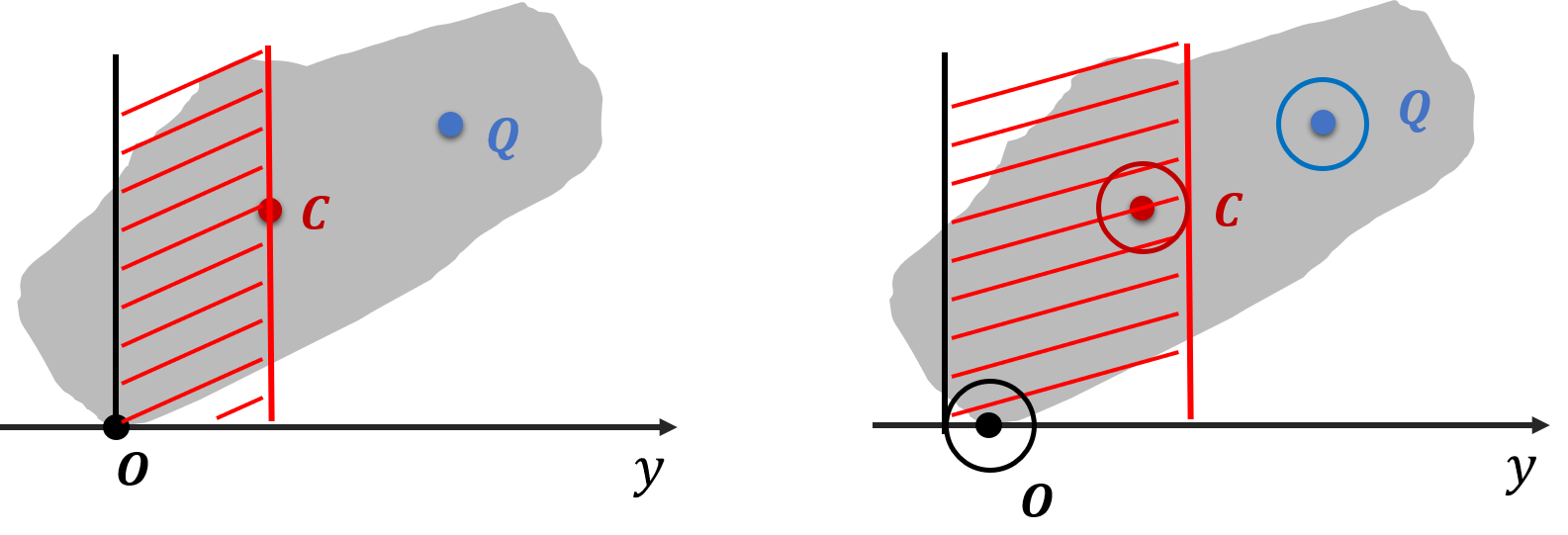}
    \caption{Illustration of the stability condition in theorem~\ref{thm:stability}. Left figure shows the condition without considering any modeling uncertainties. The point $Q$ need to be outside of the red shaded region. Right figure shows the conservative condition that works with uncertainties in $Q$, $O$ and $C$. The range of possible $Q$ should not intersect with the red shaded region. }
    \label{fig:pivoting_stability_robust}
\end{figure}

Note that if we only choose to do rolling, the method reduces to a pick-and-place reorienting. This is why the solutions of our method is a super set of the pick-and-place solutions.

\subsubsection{\bf Plan for gripper rotation}
With the object rotation trajectory fixed, the orientation of the gripper has one degree of freedom: we need to compute the rotation of the gripper about the grasp axis. Denote $\alpha_i$ as the rotation angle from the $Z$ axis of the pivoting plane frame to the $Z$ axis of the grasp frame at time step $i$, i.e. when $\alpha_i=0$, the tilting angle of the gripper is at the minimal. We need to compute the trajectory $\alpha = [\alpha_1, \dots, \alpha_N]$.

The rotation of the gripper has several constraints. First of all, the 3D orientation of the gripper must satisfy the tilting cone limit. The intersection between the pivoting plane and the 3D cone with gripper tilting angle limit $\theta_{\max}$ as half aperture gives a 2D cone on the pivoting plane, which forms a (symmetric) upper and lower bound on $\alpha$: $-\beta_{\rm cone} \le \alpha \le \beta_{\rm cone}$. The bound may be different at each time step, since the grasp axis is rotating with the object.

The second constraint comes from collision avoidance between the gripper and the object. Any part of the gripper except the fingertips should not contact the object during the robot motion.  To efficiently check the collision between the object and the gripper, we compute for each grasp $^Og\in\G$ the range of collision-free angle in the object frame during off-line computation. The constraint provides also an upper and a lower bound on $\alpha$ at each time step: $\beta_{\rm collision}^{(\min)} \le \alpha \le \beta_{\rm collision}^{(\max)}$.

The third constraint comes from the fact that the gripper has to rotate along with the object when it is not pivoting. Denote ${\bf S}\in\mathbb{R}^{N-1\times N-1}$ as a selection matrix with only 1 and 0 on its diagonal, it selects the entries of $\alpha_{i+1} - \alpha_{i}$ that correspond to firm grasps. Denote ${\bf U}=[0,\ {\bf I}_{N-1}], {\bf L}=[{\bf I}_{N-1},\ 0]$, we can express the vector of $\alpha_{i+1} - \alpha_{i}$ for $i=1, ..., N-1$ as $({\bf U}-{\bf L})\alpha$. Denote $\beta_{\rm object}\in\mathbb{R}^{N-1}$ as the array of object incremental rotations measured in the pivoting plane, we can express this constraint as a linear equality on $\alpha$:
\begin{equation}
  {\bf S}({\bf U}-{\bf L})\alpha=\beta_{\rm object}
\end{equation}

We can illustrate the gripper rotation planning problem in figure~\ref{fig:gripper_planning}, which shows the gripper rotation angle $\alpha$ and all kinds of constraints at each time step. Our task is to find a path from time step one to time step $N$, while staying within the tilting angle limit (black dotted lines), avoiding collision region (red shaded areas), and follow object rotation within the rolling zones (blue shaded areas).
\begin{figure}[h]
    \centering
    \includegraphics[width=\columnwidth]{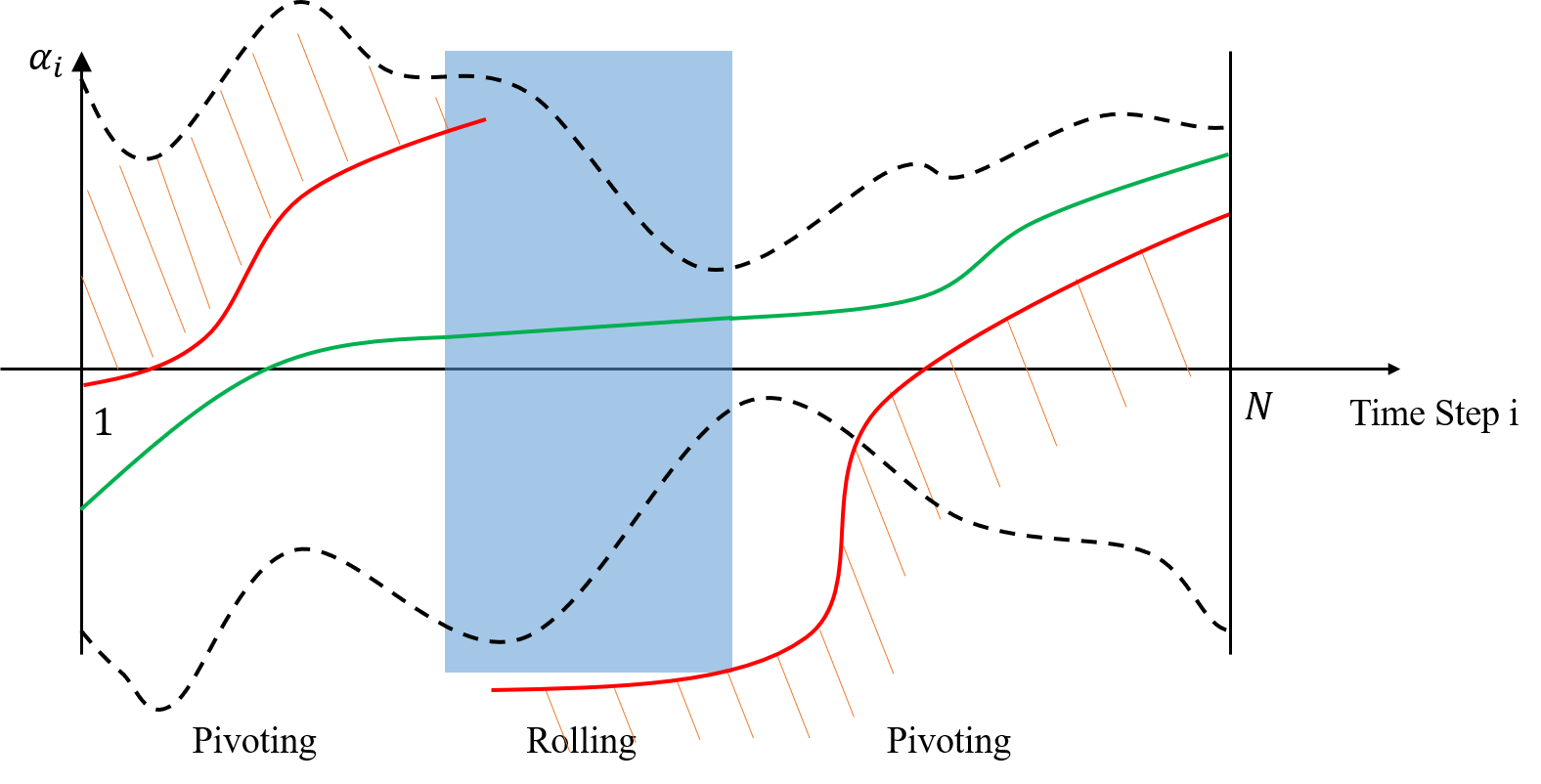}
    \caption{Illustration of the gripper rotation planning problem in the gripper-object rotation space. An example solution is shown as the green line.}
    \label{fig:gripper_planning}
\end{figure}

There are usually infinite many solutions that satisfies all the constraints. We find a unique one by applying two preferences: the solution should involve less robot motion; the gripper should stay close to the upright pose $\alpha=0$. Denote ${\bf Q}\in\mathbb{R}^{N\times N} = ({\bf U}-{\bf L})^T({\bf U}-{\bf L})$, we can formulate an optimization problem for $\alpha$:
\begin{equation}
  \label{eq:optimization_for_alpha}
  \begin{array}{*{20}{c}}
{\mathop {\min }\limits_\alpha  }&{{\alpha ^T}{\bf Q}\alpha  + k{\alpha ^T}\alpha }\\
{{\rm s}{\rm .t}{\rm .}}&{ - {\beta _{{\rm cone}}} \le \alpha  \le {\beta _{{\rm cone}}}}\\
{}&{{\beta _{{\rm collision}}^{(\min)}} \le \alpha  \le {\beta _{{\rm collision}}^{(\max)}}}\\
{}&{{\bf S}({\bf U} - {\bf L})\alpha  = {\beta _{{\rm object}}}}.
\end{array}
\end{equation}
The optimization formulation in (\ref{eq:optimization_for_alpha}) is a quadratic programing (QP), which can be solved efficiently. Finally we can compute the gripper orientation $^Wq_{{\rm grp}, i}$ for each time step $i$:
\begin{equation}
  \label{eq:compute_gripper_q_from_alpha}
  {^Wq}_{{\rm grp}, i} = {\rm aa2quat}(\alpha_i, {^Wx}_{{\rm ppf},i}) {^Wq}_{{\rm ppf}, i},
\end{equation}
where the function ${\rm aa2quat}(Angle, Axis)$ computes the quaternion for the rotation of angle $Angle$ about axis $Axis$.
${^Wx}_{{\rm ppf},i}, {^Wq}_{{\rm ppf}, i}$ denote the X axis and quaternion of the pivoting plane frame (i.e. the grasp axis) at time step $i$ respectively.
\subsubsection{\bf Plan for gripper translation}
Finally we compute the horizontal translations of the gripper, described by the horizontal velocity of the gripper frame $v_{{\rm grp}, i}\in \mathbb{R}^2$. The motion $v_{{\rm grp}, i}$ needs to satisfy three constraints. Firstly, the gripper must stay in the workspace bounding box:
\begin{equation}
  \label{eq:xy_constraint_boundingbox}
  {^Wp}^{(\min)}_{\rm xy} \le \sum_{i=1}^j {Tv_{{\rm grp}, i}} \le {^Wp}^{(\max)}_{\rm xy},\;\;\;\; j=1,\dots,N.
\end{equation}

Secondly, the motion must ends at the final gripper pose $^W\pgrp\uf$ (computed from the desired object pose):
\begin{equation}
  \label{eq:xy_constraint_final_pose}
  \sum_{i=1}^N {Tv_{{\rm grp}, i}} = ^Wp_{{\rm grp},xy}\uf - ^Wp_{{\rm grp},xy}\ui.
\end{equation}

Finally, the feasible direction of the gripper translation is limited. Denote $\delta_{OQ,i}\in\mathbb{R}^2$ as the change of gripper position in XY with respect to the contact $O$ computed from object rotations, we have the relation between $v_{{\rm grp}, i}$ and contact velocity $v_{O,i}$:
\begin{equation}
  \label{eq:velocity_Q_and_O}
  Tv_{{\rm grp}, i} = Tv_{O,i} + \delta_{OQ,i},\;\;\;\;\forall i.
\end{equation}
From theorem~\ref{thm:sliding} we know there are three possible situations for constraints on $v_{O,i}$:
\begin{enumerate}
  \item (Rolling) The object may not be able to slide.
  \item (Pivoting) The object can slide such that $O$ moves towards $Q$.
  \item (Pivoting) The object can slide in all directions.
\end{enumerate}
Denote the set of time step $i$s for each situation as $\mathbb{S}_1, \mathbb{S}_2$ and $\mathbb{S}_3$. For $i\in\mathbb{S}_1$, the gripper motion should make sure the contact $O$ does not move in the world frame:
\begin{equation}
  \label{eq:xy_constraint_sticking}
  v_{O,i} = v_{{\rm grp}, i} - \delta_{OQ,i}/T = 0,\;\;\;\;\forall i\in\mathbb{S}_1.
\end{equation}
For $i\in\mathbb{S}_2$, we have a half plane constraint on $v_{O,i}$. For the sake of being robust to modeling errors, we shrink the half plane into a cone:
\begin{equation}
  \begin{array}{rl}
    \xi\overrightarrow{OQ}_i^Tv_{O,i} &\ge \overrightarrow{OT}_i^Tv_{O,i}, \\
    \xi\overrightarrow{OQ}_i^Tv_{O,i} &\ge -\overrightarrow{OT}_i^Tv_{O,i}.
  \end{array}
\end{equation}
Here $\xi$ is a positive constant for adjusting the shape of the cone, $\overrightarrow{OQ}_i$ is an unit vector that parallels the $Y$ axis of the pivoting plane frame and points from $O$ towards $Q$. $\overrightarrow{OT}_i$ is an unit vector in the table plane, and is vertical to $\overrightarrow{OQ}_i$. Rewrite this constraint in terms of $v_{{\rm grp}, i}$ using equation~\ref{eq:velocity_Q_and_O}:
\begin{equation}
  \label{eq:xy_constraint_cone}
  \begin{array}{ll}
    (\xi\overrightarrow{OQ}_i - \overrightarrow{OT}_i)^T(v_{{\rm grp}, i} - \delta_{OQ,i}/T)&\ge 0, \forall i\in\mathbb{S}_2, \\
    (\xi\overrightarrow{OQ}_i + \overrightarrow{OT}_i)^T(v_{{\rm grp}, i} - \delta_{OQ,i}/T)&\ge 0, \forall i\in\mathbb{S}_2.
  \end{array}
\end{equation}
For $i\in\mathbb{S}_3$, there are no additional constraints.

Denote $v_{\rm grp} \in \mathbb{R}^{2N}$ as a concatenated vector of all the $v_{{\rm grp}, i}$.
We find a unique solution of $v_{\rm grp} \in \mathbb{R}^{2N}$ by minimizing the gripper motion distance:
\begin{equation}
  \label{eq:xy_cost_function}
  {\mathop {\min }\limits_{v_{\rm grp}}}\;\;\;\; \sum_{i=1}^{N} v_{{\rm grp}, i}^Tv_{{\rm grp}, i},
\end{equation}
subject to the constraints (\ref{eq:xy_constraint_boundingbox})(\ref{eq:xy_constraint_final_pose})(\ref{eq:xy_constraint_sticking}) and (\ref{eq:xy_constraint_cone}). This is again a QP problem which can be solved efficiently.

The complete procedure for solving reorienting given one grasp location is summarized in algorithm~\ref{alg:planning_one_grasp}.

\begin{algorithm}[h]
\caption{Reorienting planning given one grasp}  \label{alg:planning_one_grasp}
\begin{algorithmic}[1]
\renewcommand{\algorithmicrequire} {\textbf{Input :} }
\REQUIRE $N, {^Wq}_{\rm obj}\ui, {^Wq}_{\rm obj}\uf, {^Wp}_{\rm obj}\ui, {^Wp}_{\rm obj}\uf, ^O\pfingerl, ^O\pfingerr$.
\STATE Do a SLERP between ${^Wq}_{\rm obj}\ui, {^Wq}_{\rm obj}\uf$ to compute ${^Wq}_{{\rm obj},i}, \;\; i=1, \dots, N$.
\FOR {each $i\in\{1,\;...\;N\}$}
\IF {Any workspace constraint is violated by ${^Wq}_{{\rm obj},i}$}
\RETURN Infeasible.
\ENDIF
\STATE Compute pivoting plane frame orientation $^Wq_{{\rm ppf},i}$ from $^O\pfingerl, ^O\pfingerr, {^Wq}_{{\rm obj},i}$
\STATE Compute $C$ and $Q$ in the pivoting plane.
\STATE Choose motion primitive $m_i$ using theorem~\ref{thm:stability}.
\STATE Compute the tilting angle limit in pivoting plane $\beta_{{\rm cone}, i}$.
\STATE Compute the collision-free region in pivoting plane $\beta_{{\rm collision}, i}^{(\min)}, \beta_{{\rm collision}, i}^{(\max)}$.
\IF {$[-\beta_{{\rm cone}, i}, \beta_{{\rm cone}, i}] \cap [\beta_{{\rm collision}, i}^{(\min)}, \beta_{{\rm collision}, i}^{(\max)}] = \emptyset$}
\RETURN Infeasible.
\ENDIF
\STATE Compute incremental motion $\beta_{{\rm object}, i}$ and $\delta_{OQ,i}$.
\STATE Compute $\overrightarrow{OQ}_i$, then $\overrightarrow{OT}_i=\overrightarrow{OQ}_i\times [0\;0\;1]^T$.
\ENDFOR
\STATE Solve QP problem (\ref{eq:optimization_for_alpha}) for gripper angle trajectory $\alpha$. Return `infeasible' if there is no solution.
\STATE Compute gripper 3D orientation trajectory from $\alpha$ using equation~(\ref{eq:compute_gripper_q_from_alpha}), with $^Wx_{\rm ppf}$ computed from $^Wq_{\rm ppf}$.
\STATE Compute $v_{\rm grp}$ by solving the QP problem (\ref{eq:xy_constraint_boundingbox})(\ref{eq:xy_constraint_final_pose})(\ref{eq:xy_constraint_sticking})(\ref{eq:xy_constraint_cone}) and (\ref{eq:xy_cost_function}). Return `infeasible' if no solution.
\STATE Compute $^W\pgrp$ by integrating $v_{\rm grp}$.
\RETURN $^W\qgrp$, $^W\pgrp$, $m$.
\end{algorithmic}
\end{algorithm}
%

\subsection{Search for A Sequence of Reorientings} 
\label{sub:search_for_a_sequence_of_regrasps}
Sometimes we cannot reorient the object to the goal pose using only one grasp location. In such case it is necessary to place the object down on a stable placement and reorient again with a different grasp location, and repeat this process when necessary. We find such multi-step reorienting plans by building and searching on a graph of object stable placements, we call it \textit{Stable Placements Graph}.

\begin{figure}[h]
    \centering
    \includegraphics[width=0.45\textwidth]{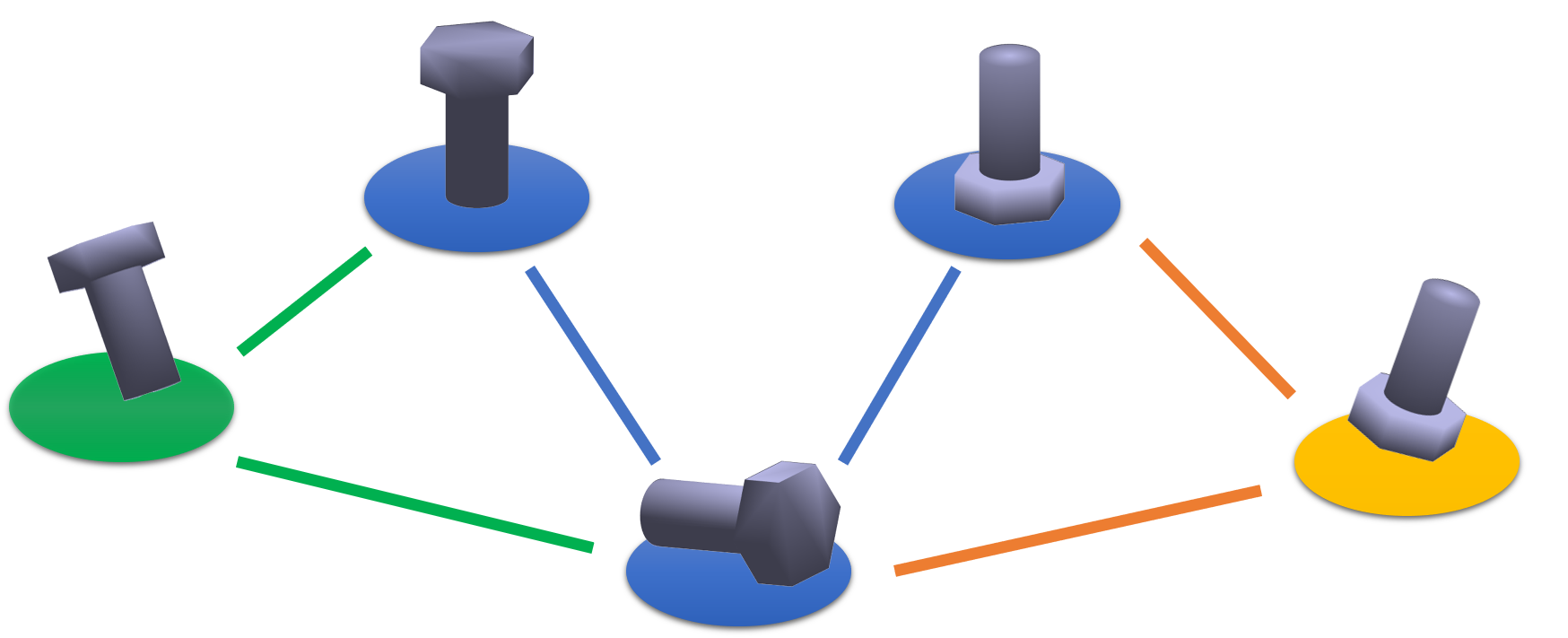}
    \caption{Example of the stable placement graph. For a screw, there are eight stable placements. Here we show three of them as three blue nodes. Given a reorienting problem, two nodes (green and yellow) are added to the graph representing the initial and goal object poses. Each edge means the two nodes share some common grasp locations. }
    \label{fig:stable_placement_graph}
\end{figure}

Our graph search idea is built upon Wan's \textit{Regrasp Graph}\cite{wan2015improving}, in which a graph connecting different object-gripper poses is computed off-line. The graph was then used for searching pick-and-place based motion plans. With pivoting being available, our stable placements graph improves Wan's work in two ways:

Firstly, each node on our graph corresponds to a unique object stable placement instead of a specific grasp pose. We no longer need to discretize gripper angles, thus remove the second layer in Wan's \textit{Regrasp Graph}. We end up with a much smaller graph. We connect two nodes in the stable placements graph as long as they have common grasp locations, rather than common grasp poses.

Secondly, we add the user specified initial and final object poses to the graph as the starting and ending node during planning time. Since our graph is very small, it takes no time to check the connection between the new nodes and the existed nodes. So this step adds little time to the online computation.

Algorithm \ref{alg:offline} summarizes the off-line computation for creating the graph.
$\rm MeshSim$ denotes the mesh simplification described in Section \ref{sub:simplified_mesh_model}.
$\rm GraspSample$ and $\rm GetStable$ are the same as in \cite{wan2015improving}: $\rm GraspSample$ samples anti-podal grasp locations on the object surface; $\rm GetStable$ finds stable placements of the object by checking whether the gravity projection is inside of the support polygon for each surface on its convex hull.

\begin{algorithm}[ht]
\caption{Off-line Computation For Each Object} \label{alg:offline}
\begin{algorithmic}[1]
\renewcommand{\algorithmicrequire} {\textbf{Input :} }
\REQUIRE Object mesh model $\bf M$
\STATE $\P \gets {\rm convexHall}({\bf M})$
\STATE ${\hat \P} \gets {\rm MeshSim}({\P})$
\STATE Sample grasp points: $\G \gets {\rm GraspSample}({\bf M})$
\STATE Compute stable placements: ${\V} \gets {\rm GetStable}({\hat \P})$
\FOR {each $g\in \G$}
  \STATE Compute collision-free angles for grasp $g$
\ENDFOR
\FOR {each $v\in \V$}
  \STATE Find the set of feasible grasp positions $\G_v\in\G$
\ENDFOR
\STATE Compute the connectivity matrix $\bf A$ for $\V$.
\RETURN $\G$, $\V$, $\bf A$.
\end{algorithmic}
\end{algorithm}

The searching algorithm for multi-step reorienting plans is shown in Algorithm \ref{alg:graph_search}.
Function $\rm OneGrasp$ denotes our planning method for one grasp location (Section \ref{sub:planning_with_one_grasp}).

We use Dijkstra to find the shortest path on the graph, minimizing the number of grasp changes.
For each edge on the path, if the destination node is not the goal node, we can choose $^Wp_{\rm obj}^{(f)}$ and $^Wq_{\rm obj}^{(f)}$ freely for that node. To encourage efficient solutions and stay within workspace constraint, we set $^Wp_{\rm obj}^{(f)}$ to be the center of the workspace bounding box, pick $^Wq_{\rm obj}^{(f)}$ as the closest orientation for that stable placement. Then we run algorithm~\ref{alg:planning_one_grasp} for each available grasps and pick the one with shortest gripper motion. If none of the grasps has a solution, we remove this edge from the stable placements graph and run graph search again.

\begin{algorithm}[ht]
\caption{Online Searching} \label{alg:graph_search}
\begin{algorithmic}[1]
\renewcommand{\algorithmicrequire} {\textbf{Input :} }
\REQUIRE $\qobj\ui,\qobj\uf,\pobj\ui,\pobj\uf, \G, \V, {\bf A}$
\STATE ${\V} \gets \{\V, v\ui, v\uf\}$, where $v\ui=\qobj\ui, v\uf=\qobj\uf$.
\STATE Find feasible grasp $\G\ui, \G\uf\in\G$ for new nodes.
\STATE Update $\bf A$ for $\V$ based on common grasps.
\WHILE{ true}
  \STATE ${\rm path}\gets {\rm Dijkstra}({\bf A}, v\ui, v\uf)$
  \IF {${\rm path} = \emptyset$}
    \RETURN $\emptyset$
  \ENDIF
  \STATE ${\rm plan} \gets \emptyset$
  \FOR {each edge $e_{jk}\in {\rm path}$}
    \STATE $\G_{jk}\gets\G_j\cap\G_k$,
    \STATE ${\rm plan}_{jk}\gets\emptyset$.
    \FOR {each $g\in \G_{jk}$}
      \STATE ${\rm sol} \gets{\rm OneGrasp}(g,\qobj\uj,\qobj\uk,\pobj\uj,\pobj\uk)$.
      \STATE ${\rm plan}_{jk} \gets \{{\rm plan}_{jk}, sol\}$.
    \ENDFOR
    \IF {${\rm plan}_{jk}\ne\emptyset$}
      \STATE ${\rm plan} \gets [{\rm plan}\ {\rm plan}_{jk}]$
    \ELSE
      \STATE Remove $e_{jk}$ from $\bf A$
      \STATE ${\rm plan} \gets \emptyset$ and \textbf{break}
    \ENDIF
  \ENDFOR
  \IF {${\rm plan} \ne \emptyset$}
    \RETURN  ${\rm plan}$.
  \ENDIF
\ENDWHILE
\end{algorithmic}
\end{algorithm}


\section{HARDWARE IMPLEMENTATION} 
\label{sec:hardware_implementation}
\subsection{Gripper Design}
The analysis above has the following assumptions on the gripper hardware:
\begin{itemize}
	\item The gripper can switch between firm grasp and pivoting;
	\item In pivoting, the object only rotates about the grasp axis without tangential slipping.
\end{itemize}
These requirements can be met by installing customized fingers on off-the-shelf parallel grippers or pinch grippers, making a ``two-phase gripper'' (\cite{chavan2015two-phase,terasaki1998motion}). We propose a simple finger design for two-phase grippers, as shown in figure~\ref{fig:two_phase_gripper}. Each of the finger has a rotational shaft installed on a ball bearing parallel to the grasp axis. The shaft has two degrees of freedom: rotation about the axis, and translation along the axis. A round, thin piece of high friction rubber is installed on the shaft as the fingertip, which moves along with the shaft. The other side of the fingertip is also stuffed with rubber. This side could touch the rough surface (filing paper) on the finger and create a high friction contact. When the finger is open, the shaft is pushed away from the rough surface from the other end by a spring-levering mechanism.

\begin{figure}[t]
    \centering
    \includegraphics[width=\columnwidth]{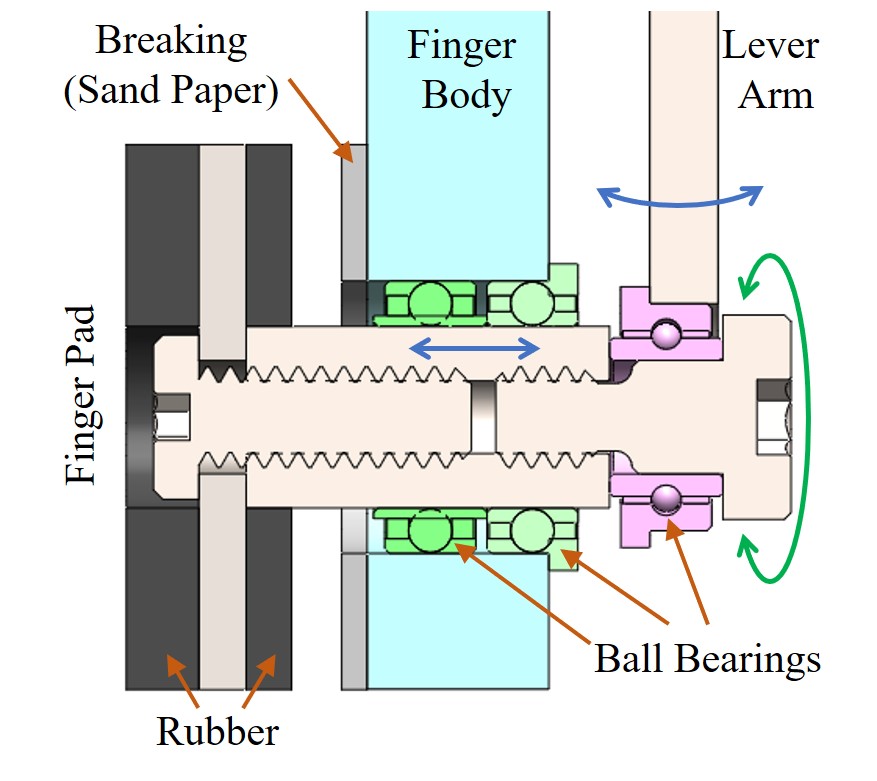}
    \includegraphics[width=\columnwidth]{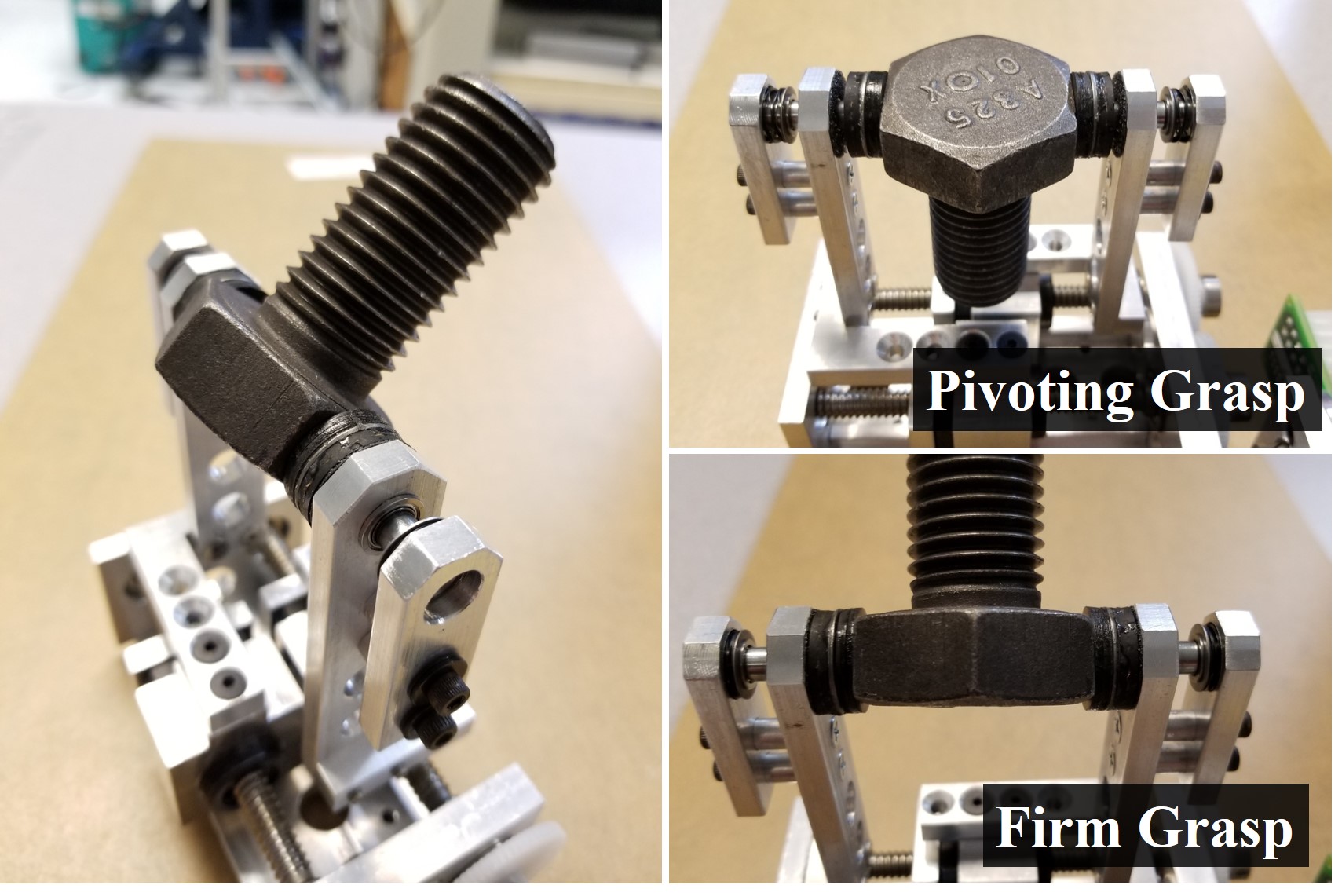}
    \caption{Our customized gripper. Upper: the structure of the two-phase finger. Lower: photo of the gripper grasping an object.}
    \label{fig:two_phase_gripper}
\end{figure}

The structures described above form a rotation mechanism and a braking system. When the finger grasps an object, the object can rotate with the shaft freely. The rotational friction is low due to the ball bearing. If the gripper continues to close the finger, the shaft will be pushed off until the back of the fingertips touches the rough surface, at which point the rotation is braked down.

Similar to existed two-phase gripper designs (\cite{chavan2015two-phase,terasaki1998motion}), our gripper achieve grasp mode switching by changing grasp force without another motor. The spring maps grasp force to grasp width. A simple control strategy is to first grasp firmly and record the grasp width, then the gripper can switch to pivoting by open the finger by half of the shaft travel.

Our design is unique in that it maintains the same contact patches on the object during each grasp mode and the switching between modes. The contact is always sticking. Without change of contact patches or contact modes, we eliminate a source of uncertainty.

One important characteristic of two-phase grippers is the range of gripping force for the rotation mode. The force has a lower bound $f_{\rm g}^{(\min)}$ to avoid tangential slip; and a upper bound $f_{\rm g}^{(\max)}$ for switching to firm grasp mode. The range $f_{\rm g}^{(\max)} - f_{\rm g}^{(\min)}$ should not be too small, otherwise the actual gripper mode will be sensitive to noises in gripper force. In our design, we make $f_{\rm g}^{(\min)}$ low by choosing high friction material for the fingertips, make $f_{\rm g}^{(\max)}$ high by choosing springs with high stiffness.

To reorient an object, it is important that the gripper can approach many different grasp locations. To approach the grasp locations that are close to the table, the overall finger width in the grasp axis direction must be as small as possible. Comparing with our previous design \cite{Hou2018FastPlanning}, we reduce the fingertip width by adopting a lever system for the spring, instead of placing the spring directly on the grasp axis.

\subsection{Execute Reorienting with Hybrid Force-velocity Control}
During roll-on-support, the object pose is over-constrained by the table and the gripper. During pivot-on-support, the object pose is also over-constrained if the contact on the table is sticking. To avoid crushing the object, we can not control the velocity of all six joints of the robot simultaneously.

Instead, we adopt hybrid force-velocity control to execute the two motion primitives. During rolling, the gripper performs force control in the Z direction with a certain force to maintain the contact between the table, while all the other five DOFs (3D orientation and XY translation) are executed exactly by velocity control. Together with one constraint on the Z direction from the table contact (there is no XY constraint since the normal force is limited by force control), the object pose is determined without conflicts.

During pivoting, the situation is more complicated. If the table contact is sliding, the gripper should only perform velocity control. In this way the gripper imposes five constraints on the object, together with the table Z constraint they uniquely determine the object pose. If the table contact is sticking, we do force control in Z direction instead.

Our hybrid force-velocity control is implemented on a position-controlled industrial robot arm with wrist mounted force-torque sensor \cite{Maples1986Experiments}.



\section{Evaluation} 
\label{sec:evaluation}

\subsection{Simulation and Comparison with Pick-and-place} 
\label{sub:simulation_and_comparison_with_pick_place}
We simulate reorienting tasks on 12 objects with non-trivial shapes (over 2000 facets per object on average) obtained from Dex-Net \cite{mahler2017dex}, as shown in figure~\ref{fig:objects}. In off-line computation, we sample at most 50 grasp positions for each object. After trimming similar grasps, we keep 10 to 40 grasps for planning.

\begin{figure}[h]
    \centering
    \includegraphics[width=0.4\textwidth]{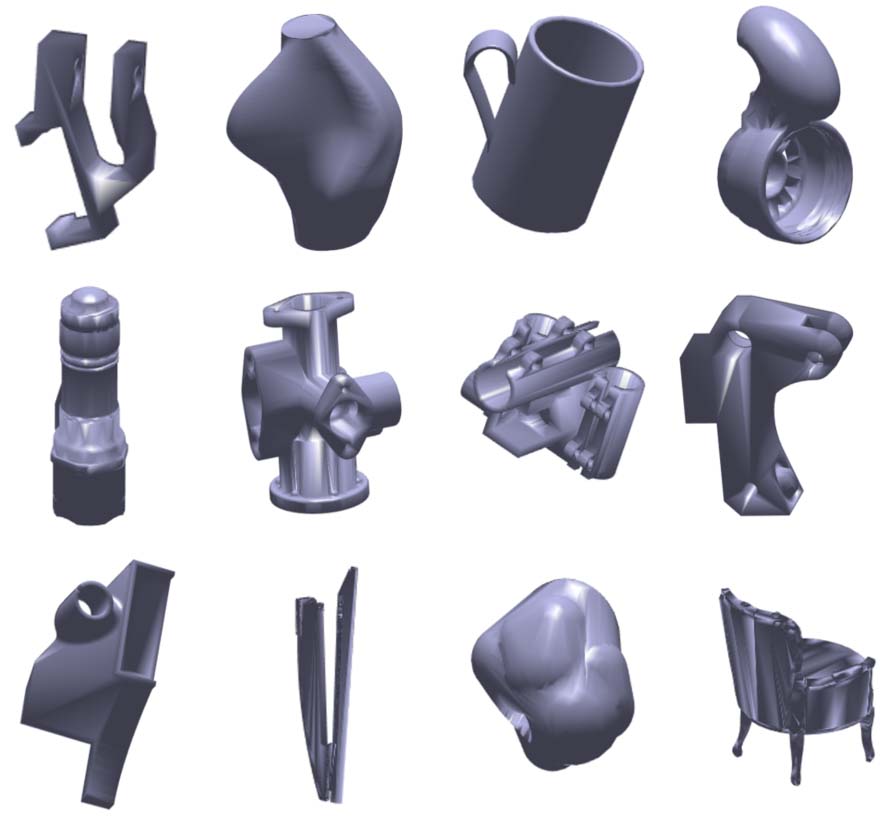}
    \caption{The 12 objects used in our simulation.}
    \label{fig:objects}
    \vspace{-0.3cm}
\end{figure}

The workspace bounding box is a $300mm\times300mm$ rectangle region. We scale each object to fit into an $100mm\times100mm\times100mm$ cube.
For each object, we create 100 reorienting problems by sampling 100 pairs of initial and final object poses with feasible grasp locations.

To evaluate the algorithm performance under different workspace constraints, we run the 1200 sample problems multiple times, with tilting angle limit $\theta_{\max}$ ranging from 10 to 80 degrees. For comparison, we also implement a pick-and-place based method by always avoiding pivoting in algorithm~\ref{alg:planning_one_grasp}. The performance of our method and pick-and-place are shown in figure~\ref{fig:simulations}. Notice that our method can solve more problems under all conditions (Figure~\ref{fig:simulation_result1}.

Figure~\ref{fig:simulation_result2} shows the average execution time of the solved motion plans generated by both methods. The robot end-effector maximum velocity is limited by $0.1{\rm m/s}$ for translation and $35{\rm deg/s}$ for rotation, under which our motion plan takes around seven seconds to execute (the green line). Notice that results of our method include solutions for more challenging problems. To make a fair comparison, we also show the average execution time for the problems that can be solved by both methods, shown by the orange line and the blue line, respectively. Our method is slightly more efficient comparing with pick-and-place.

Figure~\ref{fig:simulation_result3} shows the average computation time taken by each method to solve a reorienting problem. We haven't optimize the code for speed; the data are measured on a desktop with Intel Xeon 3.10GHz CPU running single-thread Matlab.
The average computation times for solving one problem (or declare failure) are $1.8{\rm s}$ and $0.96{\rm s}$ for our method and pick-and-place, respectively. The off-line computation described in algorithm~\ref{alg:offline} takes several minutes per object, depending on the complexity of the object mesh.

\begin{figure}[t]
    \centering
    \begin{subfigure}[b]{\columnwidth}
    \centering
    \includegraphics[width=\textwidth]{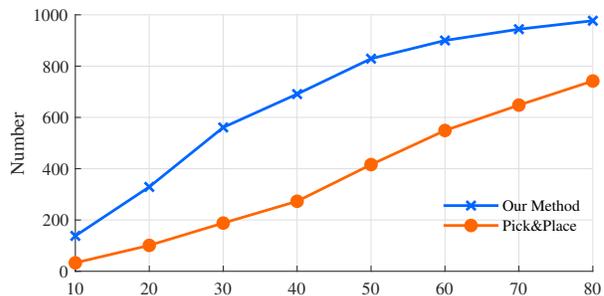}
    \caption{Number of solved problems. There are 1200 problems in total.}
    \label{fig:simulation_result1}
    \end{subfigure}
    \begin{subfigure}[b]{\columnwidth}
    \centering
    \includegraphics[width=\textwidth]{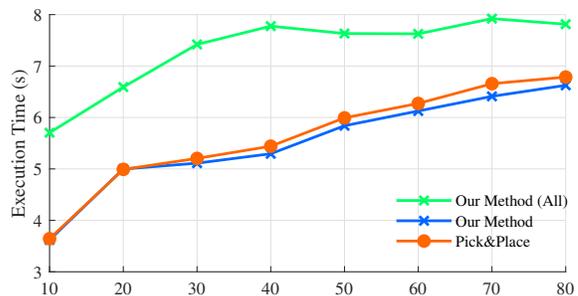}
    \caption{Average length of motion plans for solved problems, described by execution time. Green and orange lines show the results for our method and pick-and-place method, respectively. Blue line shows our method, but only counts the problem that can also be solved by pick-and-place. }
    \label{fig:simulation_result2}
    \end{subfigure}
    \begin{subfigure}[b]{\columnwidth}
    \centering
    \includegraphics[width=\textwidth]{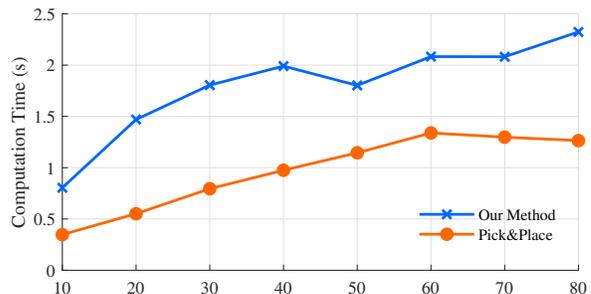}
    \caption{Average computation time per solved problem.}
    \label{fig:simulation_result3}
    \end{subfigure}
    \caption{Simulation results. In all the sub-figures, the horizontal axes are the tilting angle limits in degrees.}
    \label{fig:simulations}
\end{figure}

\subsection{Experiments} 
\label{sub:experiments}
We test our method with an ABB IRB120 industrial robot and the two-phase gripper described in \ref{sec:hardware_implementation}. To implement multi-step motion plans, we obtain object 3D pose feedback from vision before each grasp. The vision system includes two Intel RealSense D415 RGBD cameras. There is no vision feedback when the fingers are on the object.

\begin{figure}[t]
    \centering
    \includegraphics[width=0.9\columnwidth]{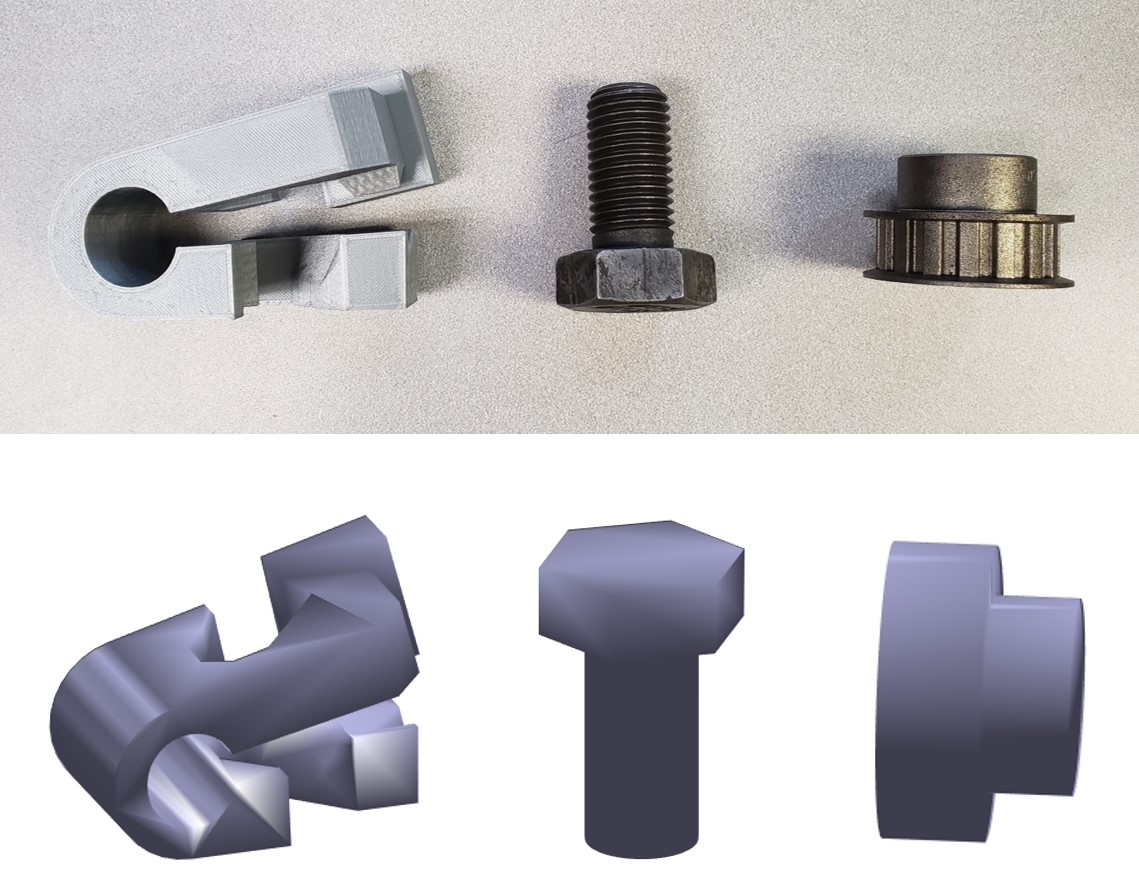}
    \caption{First row: objects used for experiments. Second row: corresponding mesh model used for planning.}
    \label{fig:experiment_objects}
\end{figure}

Figure~\ref{fig:experiment_objects} shows the objects we used in experiments, including a metal screw, a metal pulley wheel and a 3D printed clamp from DexNet dataset. For the two real objects, we use rough models for perception and planning.
Solutions of one example problem for each object are shown in figure~\ref{fig:experiment_1},\ref{fig:experiment_2} and \ref{fig:experiment_3}.
More experiments can be found in the supplementary video.

\begin{figure*}[t]
    \centering
    \includegraphics[width=\textwidth]{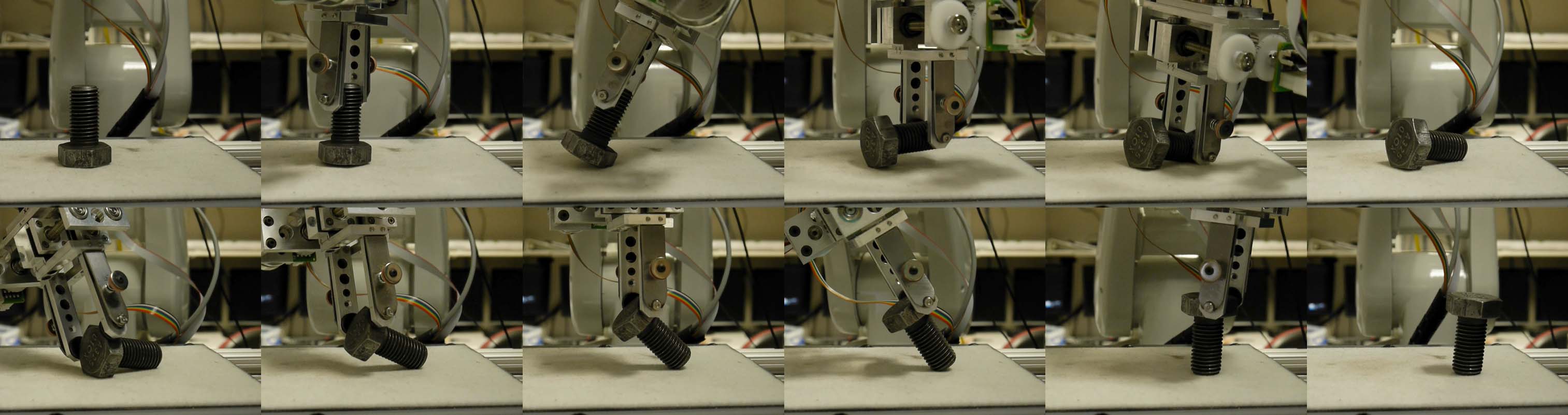}
    \caption{Snapshots of reorienting a screw. The motion plan contains two steps, as shown in the two rows respectively.}
    \label{fig:experiment_1}
\end{figure*}

\begin{figure*}[t]
    \centering
    \includegraphics[width=\textwidth]{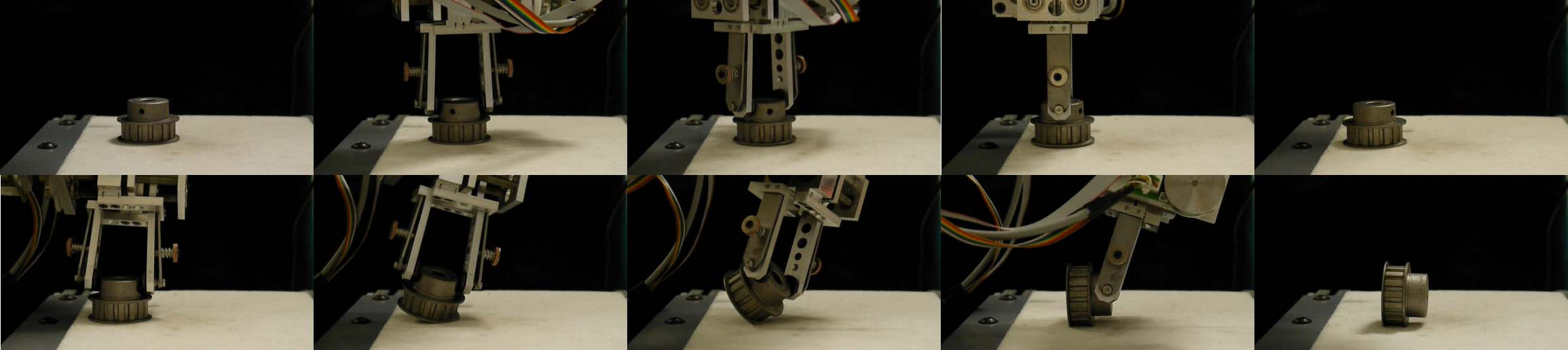}
    \caption{Snapshots of reorienting a pulley wheel. The motion plan contains two steps, as shown in the two rows respectively.}
    \label{fig:experiment_2}
\end{figure*}

\begin{figure*}[t]
    \centering
    \includegraphics[width=\textwidth]{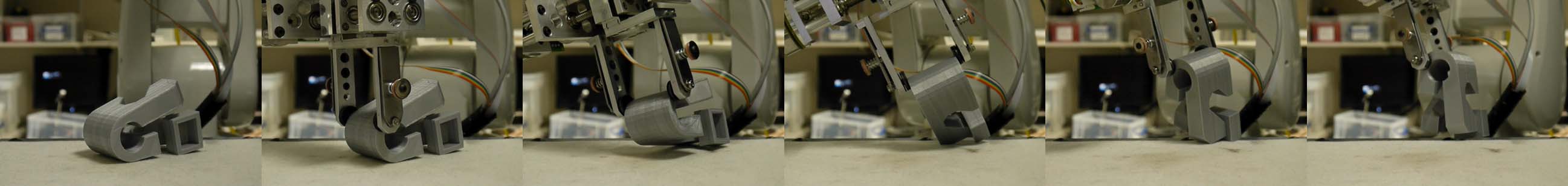}
    \caption{Snapshots of reorienting a plastic clamp.}
    \label{fig:experiment_3}
\end{figure*}

\subsection{Failures and future work} 
The experiment can fail in several ways. The most common type of failure is unexpected slipping at the grasp location, which is assumed not to happen in planning. The problem can be fixed by increasing the maximum resistance force at the finger (use stronger gripper, higher frictional fingertips), or considering tangential force limit constraint in planning. It's also helpful to implement better force control, since the large tangential force is often caused by the variations in force tracking.

The other main reason for failure is wrong estimation of the friction coefficient between the object and the table. When our estimation is off too much, theorem~\ref{thm:stability} fails to predict the stability of pivoting. After we tune and obtain a more accurate friction coefficient, this problem no longer happens. In the future, an online estimation algorithm for updating friction parameter and closed-loop control may solve this issue.

\section*{Acknowledgment}
The authors would like to thank Nikil Chavan-Dafle, Fran{\c{c}}ois Hogan, Mohamed Raessa and Weiwei Wan for helpful discussions.

\ifCLASSOPTIONcaptionsoff
  \newpage
\fi

\bibliographystyle{plain}
\bibliography{refs}

\end{document}